\documentclass{article}

\usepackage{arxiv}
\usepackage{graphicx}

\usepackage[utf8]{inputenc} 
\usepackage[T1]{fontenc}    
\usepackage{hyperref}       
\usepackage{url}            
\usepackage{booktabs}       
\usepackage{amsfonts}       
\usepackage{amssymb,amsmath,array,amsthm}
\usepackage{microtype}      
\usepackage{lipsum}
\usepackage{natbib}

\DeclareMathOperator{\softmax}{softmax}
\newtheorem{thm}{Theorem}
\newtheorem{lemma}[thm]{Lemma}

\title{Fourier Neural Networks: A Comparative Study}

\author{
  Abylay Zhumekenov,  Malika Uteuliyeva, Rustem Takhanov, Zhenisbek Assylbekov, Alejandro J. Castro \\
  Department of Mathematics, Nazarbayev University
   \And
 Olzhas Kabdolov \\
  BTS Digital
}

\begin{document}
\maketitle

\begin{abstract}
We review neural network architectures which were motivated by Fourier series and integrals and which are referred to as Fourier neural networks. These networks are empirically evaluated in synthetic and real-world tasks. Neither of them outperforms the standard neural network with sigmoid activation function in the real-world tasks. All neural networks, both Fourier and the standard one, empirically demonstrate  lower approximation error than the truncated Fourier series when it comes to  approximation of a known function of multiple variables.
\end{abstract}

\section{Introduction}

Over the past few years, neural networks have re-emerged as powerful machine-learning models, yielding state-of-the-art results in fields such as computer vision, speech recognition, and natural language processing. In this work we explore several neural network architectures, the authors of which were inspired by Fourier series and integrals. Such architectures will be collectively referred to as \textit{Fourier Neural Networks} (\textit{FNNs}). First FNNs were proposed in 80s and 90s, but they are not widely used nowadays. Is there any reasonable explanation for this, or were they  simply not given enough attention? To answer this question we perform empirical evaluation of the FNNs, found in the existing literature, on synthetic and real-world datasets. We are mainly interested in the following hypotheses: Is any of the FNNs superior to others? Does any FNN outperform conventional feedforward neural network with the logistic sigmoid activation? 
Our experiments show that the FNN of \citet{gallant1988there} outperforms all other FNNs, and that \textit{all} FNNs are not better than  the standard feedforward neural architecture with sigmoid activation function except the case of  modeling synthetic data.

%

\section{Preliminaries}
\textbf{Notation.} We let $\mathbb{Z}$ and $\mathbb{R}$ denote the integer and real numbers, respectively. Bold-faced letters ($\mathbf{x}$, $\mathbf{y}$) denote vectors in $d$-dimensional Euclidean space $\mathbb{R}^d$, and plain-faced letters ($x$, $f$) denote either scalars or functions. 
$\langle\cdot,\cdot\rangle$ denotes inner product: $\langle\mathbf{x},\mathbf{y}\rangle:=\sum_{j=1}^dx_j y_j$; and $\|\cdot\|$, $\|\cdot\|_2$ denote the Euclidean norm: $\|\mathbf{x}\|:=\|\mathbf{x}\|_2:=\sqrt{\langle\mathbf{x},\mathbf{x}\rangle}$. 

\noindent\textbf{Feedforward Neural Networks.} Following a standard convention, we define a \textit{feedforward neural network} with one hidden layer of size $n$ on inputs in $\mathbb{R}^d$ as
\begin{equation}
\mathbf{x}\mapsto v_0 + \sum_{k=1}^n v_k\sigma(\langle\mathbf{x}, \mathbf{w}_k\rangle+b_k)\label{feedforward}
\end{equation}
where $\sigma(\cdot)$ is the activation function, and $v_k, b_k \in\mathbb{R}$, $\mathbf{w}_k\in\mathbb{R}^d$, $k = 1,\ldots, n$, are parameters of the network. The universal approximation theorem (\citet{hornik1989multilayer}; \citet{cybenko1989approximation}) states that a feedforward network (\ref{feedforward}) with any ``squashing'' activation function $\sigma(\cdot)$, such as the logistic sigmoid  function, can approximate any Borel measurable function $f(\mathbf{x})$ with any desired non-zero amount of error, provided that the network is given enough hidden layer size $n$. Universal approximation theorems have also been proved for a wider class of activation functions, which includes the now commonly used $\mathrm{ReLU}$ (\citet{leshno1993multilayer}). The  neural network \eqref{feedforward} with logistic sigmoid activation $\sigma(x):=1/(1+e^{-x})$ is referred to as \textit{standard} or \textit{vanilla feedforward neural network}.

\noindent\textbf{Fourier Series.} Let $f(\mathbf{x})$ be a function integrable in the $d$-dimensional cube $[-\pi,\pi]^d$. The \textit{Fourier series} of the function $f(\mathbf{x})$ is the series
\begin{equation}
\sum_{\mathbf{k}\in\mathbb{Z}^d}\hat{f}_\mathbf{k} e^{i\langle \mathbf{x},\mathbf{k}\rangle},\label{fourier_series}
\end{equation}
where the numbers $\hat{f}_\mathbf{k}$, called \textit{Fourier coefficients}, are defined by
$$
\hat{f}_\mathbf{k}:=(2\pi)^{-d}\underset{[-\pi,\pi]^d}{\int} f(\mathbf{y}) e^{-i\langle \mathbf{y},\mathbf{k}\rangle}d\mathbf{y},
$$
Conceptually, the feedforward neural network with one hidden layer (\ref{feedforward}) and the partial sum of the Fourier series (\ref{fourier_series}) are similar in a sense that both are linear combinations of non-linear transformations of the input $\mathbf{x}$. The major differences between them are as follows: 
\begin{itemize}
\item The Fourier series has a direct access to the function $f(\mathbf{x})$ being approximated, whereas the neural network does \textit{not} have it --- instead it is usually given a training set of pairs $\{\mathbf{x}_i, f(\mathbf{x}_i)+\epsilon_i\}$, where $\epsilon_i$ is a noise (error).
\item The coefficients and linear transformations of the input in the Fourier series are fixed, but they are trainable in the neural network and are subject to estimation based on the training set $\{\mathbf{x}_i, f(\mathbf{x}_i)+\epsilon_i\}$.
\end{itemize}
There exists a variety of results on convergence of different types of partial sums (rectangular, square, spherical) of the multiple Fourier series (\ref{fourier_series}) to $f(\mathbf{x})$ in various senses (uniform, mean, almost everywhere). We refer the reader to the works of \citet{alimov1976convergence} and \citet{alimov1977problems} for a survey of such results. It seems that the existence of such convergence guarantees has motivated several authors to design the activation functions for (\ref{feedforward}) in such a way that the resulting neural networks mimic the behavior of the Fourier series (\ref{fourier_series}). In the next section we give a brief overview of such networks.

\section{Fourier Neural Networks}\label{FNN}
\textbf{FNN of \citet{gallant1988there}}:
The earliest attempt on making a neural network resemble the Fourier series is due to \citet{gallant1988there} who have suggested the ``cosine squasher''
\begin{equation}
\sigma_{\rm GW}(x):=\begin{cases}
0, &x\in(-\infty, -\frac{\pi}{2}),\\
\frac{1}{2}\left(\cos\left(x+\frac{3\pi}{2}\right)+1\right), &x\in[-\frac{\pi}{2}, \frac{\pi}{2}],\\
1, &x\in(\frac{\pi}{2}, +\infty),
\end{cases}\label{gw_activ}
\end{equation}
as an activation function in the feedforward network (\ref{feedforward}). Moreover, they show that when additionally the connections $\mathbf{w}_i$, $b_i$ from input to hidden layer are hardwired in a special way, the obtained feedforward network yields a Fourier series approximation to a given function $f(\mathbf{x})$. Thus, such networks possess all the approximation properties of Fourier series representations. In particular, approximation to any desired accuracy of any square integrable function can be achieved by such a network, using sufficiently many hidden units. \citet{mccaffrey1994convergence} showed that the squared approximation error for sufficiently smooth functions is of order $O(n^{-1})$, where $n$ is the network's hidden layer size. We notice here that \citet{barron1993universal} has established the same order of the approximation error for the feedforward networks with any sigmoidal activation\footnote{A bounded measurable function $\phi(x)$ on the real line is called \textit{sigmoidal} if $\phi(x)\to1$ as $x\to+\infty$ and $\phi(x)\to0$ as $x\to-\infty$.} and when the function being approximated $f(\mathbf{x})$ has a bound on the first moment of the magnitude distribution of the Fourier transform. FNN of \citet{gallant1988there} is denoted as $f_{\rm GW}$.

\noindent\textbf{FNN of \citet{silvescu1999fourier}}:
Another attempt to mimic the behavior of the Fourier series by a neural network was done by \citet{silvescu1999fourier}, who introduced the following FNN:
\begin{equation}
f_{\rm S}: \mathbf{x}\mapsto v_0 + \sum_{k=1}^n v_k \sigma_{\rm S}(\mathbf{x}; \boldsymbol{\omega}_k, \boldsymbol{\phi}_k),\label{silvescu}
\end{equation}
with
\begin{equation}
\sigma_{\rm S}(\mathbf{x};\boldsymbol{\omega}_k, \boldsymbol{\phi}_k):=\prod_{j=1}^d\cos(\omega_{kj}x_j+\phi_{kj}),\label{silvescu_activation}
\end{equation}
where $\boldsymbol{\omega}_k, \boldsymbol{\phi}_k\in\mathbb{R}^d$, $v_k\in\mathbb{R}$ are trainable parameters. As we can see, Silvescu's FNN (\ref{silvescu}) does not follow the framework of the standard feedforward neural networks (\ref{feedforward}), and moreover its activation function is not sigmoidal. Figure~\ref{feedforward_vs_silvescu} depicts the difference between (\ref{feedforward}) and (\ref{silvescu}) for the case when $d=3$ and $n=2$. 
\begin{figure}[htbp]
\begin{center}
\includegraphics[width=.9\textwidth]{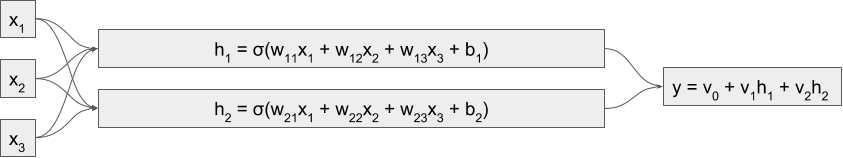}

\vspace{20pt}
\includegraphics[width=.9\textwidth]{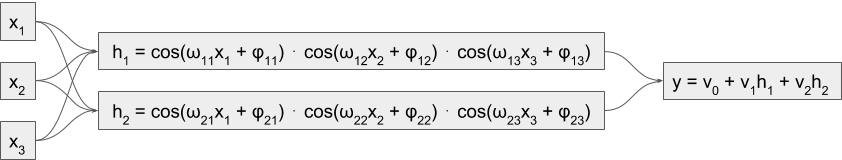}
\end{center}
\caption{Standard feedforward NN (top) vs Silvescu's Fourier NN (bottom). In the standard NN non-linearity is applied on top of the linear transformation of the \textit{whole} input, whereas in the Silvescu's network non-linearity is applied \textit{separately to each} component of the input vector.}
\label{feedforward_vs_silvescu}
\end{figure}
Because of this difference, the result of \citet{barron1993universal} is not applicable to Silvescu's FNN. However, we conjecture that the same convergence rate is valid for the Silvescu's FNN. 
The proof (or disproof) of this conjecture is deferred to our future work.

\noindent\textbf{FNN of \citet{liu2013fourier}}:
More recently, several authors suggested the following architecture
\begin{equation}
f_{\rm L}:\mathbf{x}\mapsto v_0+\sum_{k=1}^n v_k\cos(\langle\mathbf{w}_k,\mathbf{x}\rangle+b_k) + u_k\sin(\langle\mathbf{p}_k,\mathbf{x}\rangle+q_k),\label{zuo_cai}
\end{equation}
where $\mathbf{w}_k,\mathbf{p}_k\in\mathbb{R}^d$, $b_k, q_k\in\mathbb{R}$ are either hardwired or trainable, and $v_k,u_k\in\mathbb{R}$ are trainable parameters. \citet{tan2006fourier} explored aircraft engine fault
diagnostics using (\ref{zuo_cai}), \citet{zuo2005tracking}, \citet{zuo2008adaptive}, \citet{zuo2009fourier} used it for the control of a class of uncertain nonlinear systems. The above-mentioned authors did not provide rigorous mathematical analysis of this architecture, instead they used it as an ad-hoc solution in their engineering tasks. Although this FNN fits into the general feedforward framework \eqref{feedforward}, its activations are not sigmoidal, and thus the result of \citet{barron1993universal} is not applicable here as well. 
\citet{liu2013fourier} empirically evaluated (\ref{zuo_cai}) on various datasets and showed that in certain cases it converges faster
than the feedforward network with sigmoid activation and has equally good predicting accuracy and generalization ability. Also, only in the work of \citet{liu2013fourier} all the weights in \eqref{zuo_cai} are allowed to be trainable, hence we refer to this architecture as $f_{\rm L}$.

\section{Empirical Evaluation}
In this section we will perform empirical evaluation of the Fourier neural networks $f_{\rm GW}$, $f_{\rm S}$, $f_{\rm L}$ from Section~\ref{FNN} against vanilla feedforward network (\ref{feedforward}) with sigmoid activation\footnote{I.e. we put $\sigma(x):=\frac{1}{1+\exp(-x)}$ in (\ref{feedforward}).} on synthetic and real-world datasets. By ``synthetic datasets'' we mean datasets generated from a \textit{known} function. In this case we can also compare the performance of Fourier neural networks to the approximation error given by the partial Fourier series.

\subsection{Synthetic tasks}
We try to approximate a function of one variable $x\mapsto|x|$, $x\in[-\pi,\pi]$, and a function of $d=100$ variables: $\mathbf{x}\mapsto\mathbb{I}[\|\mathbf{x}\|\le1]$, $\mathbf{x}\in\{\mathbf{x}\in\mathbb{R}^{100}:\,\|\mathbf{x}\|\le2\}$, where $\mathbb{I}[\cdot]$ is the indicator function.\footnote{This means that $\mathbb{I}[\|\mathbf{x}\|\le1]=1$ if $\|\mathbf{x}\|\le1$, and $0$ otherwise.} In both cases we sampled $5\cdot 10^5$ data instances uniformly from the domains of the functions. To each instance, we associated a target value according to the target function $|x|$ or $\mathbb{I}[\|\mathbf{x}\|\le1]$. Another $10\cdot10^4$ examples were generated in a similar manner, of which $5\cdot10^4$ examples were used as a validation set, and $5\cdot10^4$ examples were used as a test set. We trained 32 networks on these datasets: for each of the above-mentioned models (vanilla feedforward network, $f_{\rm GW}$, $f_{\rm S}$, $f_{\rm L}$) we varied the hidden layer size from 100 to 800 with the step 100. Training was performed with Adam optimizer (\citet{kingma2014adam}). 
We used the squared loss $l(y,\hat{y}):=(y-\hat{y})^2$ and batches of size 100. For each model, a learning rate was tuned separately on the validation set. The results are presented in Figure~\ref{synth_data_results}. 
\begin{figure*}
\begin{center}
\includegraphics[height=0.18\textheight]{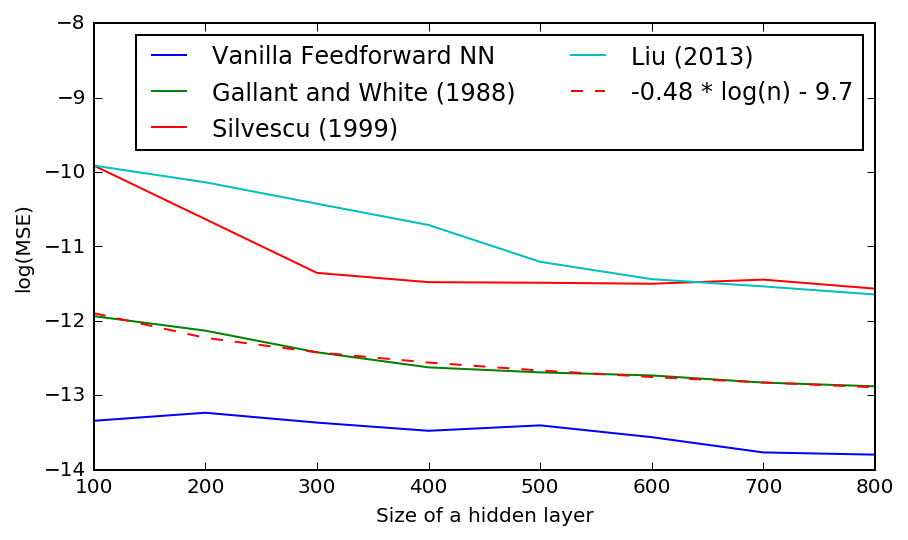}\  \includegraphics[height=0.18\textheight]{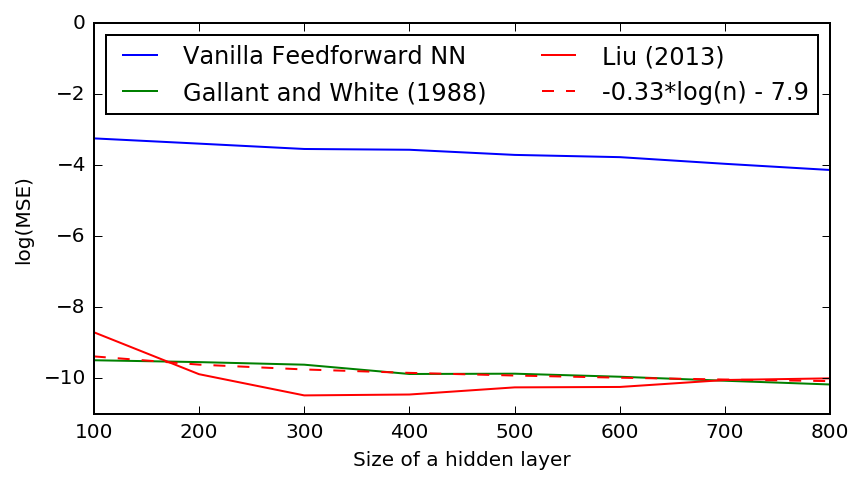}
\end{center}
\caption{Results of approximating $|x|$ (left) and $\mathbb{I}[\|\mathbf{x}\|\le1]$ (right) by Fourier neural networks and Fourier series. MSE stands for the mean squared error, $\frac{1}{T}\sum_{i=1}^T(y_i-\hat{y}_i)^2$. Dashed curves were obtained by regressing $\log(\mathrm{MSE})$ of $f_{\rm GW}$ on $\log n$. $^\ast$Evaluation of $f_{\rm S}$ for the indicator function data is in progress.}
\label{synth_data_results}
\end{figure*}
Vanilla feedforward network \eqref{feedforward} obtains lowest mean squared error (MSE) for $|x|$, whereas the FNN of \citet{gallant1988there} outperforms all other models for $\mathbb{I}[\|\mathbf{x}\|\le1]$. According to the regression fits (dashed curves in Fig.~\ref{synth_data_results}), the function $x\mapsto|x|$ is approximated by the neural networks with error $O(n^{-0.48})$, and this is much worse than the approximation error given by the partial sums of the Fourier series of $f(x)=|x|$, which, according to Lemma~\ref{abs_value_approx_lemma} below, is of order $O(n^{-3})$. For the function $\mathbf{x}\mapsto\mathbb{I}[\|\mathbf{x}\|\le1]$  results are to other way around: the approximation error by the neural networks is of order $O(n^{-0.33})$, while it is of the order $O(n^{-1/100})$ by the truncated Fourier series (see Lemma~\ref{lemma3} below). We keep in mind that the theoretical result of \citet{barron1993universal} states that for any function from a certain class\footnote{Functions with bounded first moment of the magnitude distribution of the Fourier transform, which we refer to as \textit{Barron functions} in agreement with \citet{lee2017ability}.} (to which the indicator function does belong) a feedforward neural network with one hidden layer of size $n$ will be able to approximate this function with a squared error of order $O(n^{-1})$. We are not guaranteed, however, that the training algorithm will be able to \textit{learn} that function. Even if the neural network is able to \textit{represent} the function, learning can fail, since the optimization algorithm used for training may not be able to find the value of the parameters that corresponds to the desired function. We attribute the mismatch, $O(n^{-1/3})$ instead of $O(n^{-1})$, between the orders of approximation errors to the suboptimal estimation of the parameters of the networks by the Adam optimizer. However, in general we have experimentally confirmed Barron's claim that neural networks with $n$ hidden units can approximate functions with much smaller error than series expansions with $n$ terms.

We also notice here that directly comparing neural networks with truncated Fourier series  is somewhat unfair, as these are two different categories of approximation: Fourier series serve as some theoretical reference, which is possible \textit{only} when we have access to the function being approximated.

\begin{lemma} \label{abs_value_approx_lemma} For the $2\pi$-periodic function $f(x):=|x|$, $x\in[-\pi, \pi]$, let $S_n(x)$ be the $n^{\text{th}}$ partial sum of its Fourier series. Then, for some constant $C$,
\begin{equation}
\|f-S_n\|_2^2
\le \frac{C}{n^3}.\label{abs_value_approx_error_result}
\end{equation}
\end{lemma}
\begin{proof}
The Fourier series expansion of $f$ is given by 
\begin{equation}
f(x)=\frac{\pi}{2}+\sum_{k=1}^\infty a_k \cos(2k-1)x,
\qquad 
a_k
:= -\frac4\pi \frac{1}{(2k-1)^2},
\end{equation}
(see Example 1, p. 23, from \citet{folland1992fourier}), and therefore
by Parseval's Theorem,
\begin{align}
\|f-S_n\|_2^2
& :=\int_{-\pi}^\pi(f(x)-S_n(x))^2dx 
 =\pi\sum_{k=n+1}^{\infty}{a_k^2}\notag
 \\&\le\pi\sum_{k=n+1}^{\infty}\Big(-\frac{4}{\pi}\frac{1}{(2k-1)^2}\Big)^2=\frac{16}{\pi}\sum_{k=n+1}^{\infty}\frac{1}{(2k-1)^4}\label{abs_value_approx_error}.
\end{align}
Since $(2k-1)^{-4}$ is a monotonically decreasing sequence, we have
$$
\int_{n+1}^\infty\frac{du}{(2u-1)^4}\le\sum_{k=n+1}^\infty\frac{1}{(2k-1)^4}\le\int_n^\infty\frac{du}{(2u-1)^4},
$$
that is,
\begin{equation}
\frac{1}{6(2n+1)^3}\le\sum_{k=n+1}^\infty\frac{1}{(2k-1)^4}\le\frac{1}{6(2n-1)^3}.\label{sum_est_by_integral}
\end{equation}    
Combining (\ref{abs_value_approx_error}) and (\ref{sum_est_by_integral}) we obtain (\ref{abs_value_approx_error_result}).
\end{proof}

\begin{lemma}\label{lemma3} Let $\mathbf{x}\in[-\pi,\pi]^d$ and $f(\mathbf{x})$ be the indicator function of the unit ball in $\mathbb{R}^d$, that is, $f(\mathbf{x}):=\mathbb{I}[\mathbf{x}\leq1]$. Let $S_R(\mathbf{x})$ be the truncated Fourier Series of $f(\mathbf{x})$, where $R \geq 1$ is the radius of the partial spherical summation and $n$ is the number of terms in the partial sum. Then, for some dimensional dependent constant $C_d$,  the following holds
\begin{equation}
    \|f-S_R\|^2_2\leq \frac{C_d}{n^{1/d}}.\label{ball_error}
\end{equation}
\end{lemma}
\begin{proof}
For $\mathbf{x}\in\mathbb{R}^d$, denote  $\|\mathbf{x}\|_1:=|x_1|+\ldots+|x_d|$, and $\|\mathbf{x}\|_\infty:=\max_{1\le i\le d}|x_i|$. It is known that
\begin{equation}
    \|\mathbf{x}\|_2\le\|\mathbf{x}\|_1\le\sqrt{d}\|\mathbf{x}\|_2,\quad\|\mathbf{x}\|_\infty\le\|\mathbf{x}\|_1\le d\|\mathbf{x}\|_\infty,\label{norms1}
\end{equation}
which in particular implies
\begin{equation}
    \|\mathbf{x}\|_\infty\ge\frac1d\|\mathbf{x}\|_1\ge\frac1d\|\mathbf{x}\|_2.\label{norms2}
\end{equation}
From \eqref{norms1} and \eqref{norms2} it follows that
\begin{equation}
    \{\mathbf{k}\in\mathbb{Z}^d:\,\|\mathbf{k}\|_2>R\}\subset\{\mathbf{k}\in\mathbb{Z}^d:\,\|\mathbf{k}\|_\infty>R/d\},
\end{equation}
and, therefore,
\begin{equation}
    \sum_{\|\mathbf{k}\|_2>R}\frac{1}{\|\mathbf{k}\|^{d+1}_2}
    \le\sum_{\|\mathbf{k}\|_\infty>R/d}\frac{1}{\|\mathbf{k}\|_2^{d+1}}
    \lesssim \sum_{\|\mathbf{k}\|_\infty>R/d}\frac{1}{\|\mathbf{k}\|_1^{d+1}}.\label{sums}
\end{equation}
Here ``$A\lesssim B$'' means that ``$A\le C_d B$'', for some dimensional dependent constant $C_d$. 
Analogously, we write ``$A\sim B$'' if ``$A\lesssim B$'' and ``$B\lesssim A$''.
Denoting $\tilde{R}:=R/d$, 
we obtain the following decomposition
\begin{multline*}
\{\|\mathbf{k}\|_\infty>\tilde{R}\}\\
= \bigcup_{1 \leq j \leq d} \bigcup_{1 \leq i_1 \neq \dots \neq i_d  \leq d}
\Big\{ |k_{i_\alpha}|> \tilde{R}, \, 1 \leq \alpha \leq j;\\
|k_{i_\alpha}| \leq \tilde{R}, \, j < \alpha \leq d \Big\}.    
\end{multline*}
Thus the latter sum in \eqref{sums} can be estimated as follows, 
\begin{align}
&\sum_{\|\mathbf{k}\|_\infty>\tilde{R}}\frac{1}{\|\mathbf{k}\|_1^{d+1}}
 = \sum_{\|\mathbf{k}\|_\infty>\tilde{R}}\frac{1}{(\|\mathbf{k}\|_1^{(d+1)/j})^j}
\notag\\
& \leq 
\sum_{j=1}^d 
\sum_{1 \leq i_1 \neq \dots \neq i_d  \leq d}
\sum_{|k_{i_1}| > \tilde{R}} \cdots \sum_{|k_{i_j}| > \tilde{R}}
\sum_{|k_{i_{j+1}}| \leq \tilde{R}} \cdots \sum_{|k_{i_d}| \leq \tilde{R}}
\notag\\
& \qquad \times \frac{1}{|k_{i_1}|^{(d+1)/j} \cdots |k_{i_j}|^{(d+1)/j}}\notag\\
& \lesssim \sum_{j=1}^d  \tilde{R}^{d-j} \Big(\sum_{|\ell|>\tilde{R}}\frac{1}{|\ell|^{(d+1)/j}}\Big)^j  
\sim \sum_{j=1}^d \tilde{R}^{d-j} \Big(\int_{\tilde{R}}^\infty\frac{du}{u^{(d+1)/j}}\Big)^j \notag\\
&\sim \sum_{j=1}^d \frac{\tilde{R}^{d-j}}{\tilde{R}^{d+1-j}}
\sim\frac{1}{\tilde{R}}
\sim\frac1R.\label{one_over_R}
\end{align}
Combining \eqref{sums} and \eqref{one_over_R} we get
\begin{equation}
 \sum_{\|\mathbf{k}\|_2>R}\frac{1}{\|\mathbf{k}\|^{d+1}_2}\lesssim \frac1R. \label{sum_estimate}
\end{equation}

\quad\\

Let $\mathrm{SE}$ denote the squared error in the left-hand side of \eqref{ball_error}. Then,  Parseval's Theorem allows us to write
\begin{align}
    \mathrm{SE}
    & :=\int_{[-\pi,\pi]^d} |f(\mathbf{x})-S_R(\mathbf{x})|^2 \, d\mathbf{x}=\int_{[-\pi,\pi]^d}
    \Big|\sum_{\|\mathbf{k}\|_2> R}\hat{f}_\mathbf{k}e^{i\mathbf{k}\cdot\mathbf{x}}\Big|^2d\mathbf{x}\notag\\
    &=(2\pi)^d\sum_{\|\mathbf{k}\|_2> R}|\hat{f}_\mathbf{k}|^2.\notag
\end{align}
Using the estimates of the Fourier coefficients for the indicator function of a ball (\citet{pinsky1993radial}, p.~120), and denoting $\alpha:=\|\mathbf{k}\|_2-(d-1)\pi/4$, we get
\begin{align}
    \mathrm{SE}
    &=(2\pi)^d\sum_{\|\mathbf{k}\|_2> R}
    \Big[\frac{C_d}{\|\mathbf{k}\|_2^{(d+1)/2}}
    \Big\{ \sin\alpha
    +O\Big(\frac{1}{\sqrt{\|\mathbf{k}\|_2}}\Big)\Big\}\Big]^2\notag\\
    & \sim \sum_{\|\mathbf{k}\|_2> R}\frac{1}{\|\mathbf{k}\|_2^{d+1}}\Big[\sin^2\alpha
    +2 \sin\alpha \, O\Big(\frac{1}{\sqrt{\|\mathbf{k}\|_2}}\Big)
    +O\Big(\frac{1}{\|\mathbf{k}\|_2}\Big)\Big] \notag \\
    & \lesssim \sum_{\|\mathbf{k}\|_2> R}\frac{1}{\|\mathbf{k}\|_2^{d+1}}.  \label{preprelast}
\end{align}
From \eqref{sum_estimate} and \eqref{preprelast} it follows that
\begin{equation}
    \mathrm{SE}\lesssim\frac1R.\label{prelast}
\end{equation}
The number of terms $n$ in the spherical partial sum $S_R(\mathbf{x})$ is equal to the number of integer points in the $d$-ball of radius $R$, which is, according to \citet{gotze2004lattice}, approximated by the volume of such ball up to an error $O(R^{d-2})$, i.e. 
$$
n\sim R^d.
$$
Combining this with \eqref{prelast}, we get 
$\mathrm{SE}\lesssim n^{-1/d}$.
\end{proof}


\subsection{Image recognition}
We performed evaluation of the FNNs in the image recognition task using the MNIST dataset (\citet{lecun1998gradient}), which is  commonly used for training various image processing systems. It consists of handwritten digit
images, $28 \times 28$ pixels in size, organized into 10 classes (0 to 9) with 60,000 training and 10,000 test samples. Portion of training samples was used as validation data. Images were represented as vectors in $\mathbb{R}^{784}$, hidden layer size was fixed at 64 for all networks, and  classification was done based on the softmax normalization. Training was performed with Adam optimizer (\citet{kingma2014adam}). We used the cross-entropy loss and batches of size 100. Learning rate was tuned separately for each model on the validation data. Table 2 compares classification accuracy obtained by the models. 
\begin{table}[htbp]
\setlength{\tabcolsep}{7pt}
\begin{center}
\begin{tabular}{l c c c c | c c c c}
\toprule
Model & Accuracy & Learning rate\\
\midrule
Vanilla feedforward NN  & 0.9648 & 0.0096 \\
FNN of \citet{gallant1988there} & 0.9695 & 0.0045 \\
FNN of \citet{silvescu1999fourier} & 0.9659 & 0.0134  \\
FNN of \citet{liu2013fourier} & 0.9638 & 0.0034 \\
\bottomrule
\end{tabular}
\end{center}
\label{mnist_results}
\caption{Evaluation of the networks on MNIST data.}
\end{table}
As we can see, all the networks demonstrate similar performance in this task. In fact, the differences between accuracy results are not significant across the models, Pearson's Chi-square test of independence $\chi^2_3=5.6449$, $p$-value $> 0.1$.

\subsection{Language modeling}
A statistical language model (LM) is a model which assigns a probability to a sequence of words. Below we specify one type of such models based on the structurally constrained recurrent network (SCRN) of \citet{mikolov2014learning}.

Let $\mathcal{W}$ be a finite vocabulary of words. We assume that words have already been converted into indices. 
Let $\mathbf{E}\in\mathbb{R}^{|\mathcal{W}|\times d_\mathcal{W}}$ be an input embedding matrix for words --- i.e., it is a matrix in which the $w$th row (denoted as $\mathbf{w}$) corresponds to an embedding of the word $w\in \mathcal{W}$. 
Based on word embeddings $\mathbf{w_{1:k}}=\mathbf{w_1},\ldots,\mathbf{w_k}$ for a sequence of words $w_{1:k}$, the SCRN model produces two sequences of states, $\mathbf{s_{1:k}}$ and $\mathbf{h_{1:k}}$, according to the equations\footnote{Vectors are assumed to be row vectors, which are right multiplied by matrices ($\mathbf{x}\mathbf{W}+\mathbf{b}$). This choice is somewhat non-standard but it maps better to the way networks are implemented in code using matrix libraries such as {\tt TensorFlow}.}
\begin{align}
&\mathbf{s_t} = (1-\alpha)\mathbf{w_t}\mathbf{B}+\alpha\mathbf{s_{t-1}},\label{s}\\
&\mathbf{h_t} = \sigma(\mathbf{w_t}\mathbf{A}+\mathbf{s_t}\mathbf{P}+\mathbf{h_{t-1}}\mathbf{R}),\label{h}
\end{align}
where $\mathbf{B}\in\mathbb{R}^{|\mathcal{W}|\times d_s}$, $\mathbf{A}\in\mathbb{R}^{|\mathcal{W}|\times d_h}$, $\mathbf{P}\in\mathbb{R}^{d_s\times d_h}$, $\mathbf{R}\in\mathbb{R}^{d_h\times d_h}$, $d_s$ and $d_h$ are dimensions of $\mathbf{s_t}$ and $\mathbf{h_t}$, $\sigma(\cdot)$ is the logistic sigmoid function. 
The last couple of states $(\mathbf{s_k}, \mathbf{h_k})$ is assumed to contain information on the whole sequence $w_{1:k}$ and is further used for predicting the next word $w_{k+1}$ of a sequence according to the probability distribution
\begin{equation}
\Pr(w_{k+1}|w_{1:k})=\softmax(\mathbf{s_k}\mathbf{U} + \mathbf{h_k}\mathbf{V}),\label{softmax}
\end{equation}
where $\mathbf{U}\in\mathbb{R}^{d_s\times|\mathcal{W}|}$ and $\mathbf{V}\in\mathbb{R}^{d_{h}\times|\mathcal{W}|}$ are output embedding matrices. For the sake of simplicity we omit bias terms in (\ref{h}) and (\ref{softmax}). Being conceptually much simpler, the SCRN architecture  demonstrates performance comparable to the widely used LSTM model in language modeling task (\citet{kabdolov2018reproducing}), and this is why we chose it for our experiments.


We train and evaluate the SCRN model for $(d_h, d_s)$ $\in\{(40, 10)$, $(90, 10)$, $(100, 40)$, $(300, 40)\}$ on the PTB (\citet{marcus1993building}) data set, for which the standard training (0-20), validation (21-22), and test (23-24) splits along with pre-processing per \citet{mikolov2010recurrent} is utilized. We replace $\sigma$ in \eqref{h} with $\sigma_{\rm GW}$, $\sigma_{\rm S}$, and $\sigma_{\rm L}$ defined in Section~\ref{FNN}, and we refer to such modification as Fourier layers. The choice of hyperparameters is guided by the work of \citet{kabdolov2018reproducing}, except that for the Fourier layers we additionally tune the learning rate, its decay schedule and the initialization scale over the validation split.  To evaluate the performance of the language models we use perplexity (PPL) over the test set. The results are provided in Table~\ref{lm_results}. 
\begin{table}[htbp]
\setlength{\tabcolsep}{7pt}
    \centering
    \begin{tabular}{l c c c c}
         \toprule
         Activation & (40, 10) & (90, 10) & (100, 40) & (300, 40) \\
         \midrule
         $\sigma$   & 128.0 & 118.6 & 118.7 & 120.6 \\
$\sigma_{\rm GW}$   & 132.8 & 119.6 & 120.1 & 127.9 \\
$\sigma_{\rm S}$    & 144.4 & 133.4 & 127.7 & 125.9 \\
$\sigma_{\rm L}$    & 165.7 & 139.3 & 147.5 & 156.8 \\
         \bottomrule
    \end{tabular}
    \vspace{10pt}
    \caption{Evaluation of the SCRN language models on the PTB data. Columns 2--5 correspond to different configurations of the hidden ($d_h$) and context ($d_s$) states sizes.}
    \label{lm_results}
\end{table}
As one can see, the conventional sigmoid activation outperforms all Fourier activations, and, as in the case of synthetic data, the Fourier layer of \citet{gallant1988there} is better than other Fourier layers for most of the architectures. 

\section{Discussion}
The FNNs of \citet{silvescu1999fourier} \eqref{silvescu} and of \citet{liu2013fourier} \eqref{zuo_cai} have non-sigmoidal activations, which makes their optimization more difficult. Although the activation function of $f_{\mathrm{GW}}(\cdot)$ is sigmoidal, it still underperforms the standard feedforward neural network in almost all cases. We hypothesize that this is because $\sigma_{\mathrm{GW}}$ is constant outside $\left[-\frac\pi2,\frac\pi2\right]$, while $\sigma(x)=1/(1+e^{-x})$ is never constant. This means that $\forall x_1,\,x_2\in(\pi/2,\infty)$: $\sigma_\mathrm{GW}(x_1)=\sigma_\mathrm{GW}(x_2)$, i.e. the activation of \citet{gallant1988there} \eqref{gw_activ} does not distinguish between any values to the right from $\pi/2$ (and to the left from $-\pi/2$). The standard sigmoid activation $\sigma(\cdot)$, on the other hand, can theoretically\footnote{In practice, when implemented on a computer $\sigma(\cdot)$ will also be ``constant'' outside an interval $[-a,a]$, where $a$ depends on the type of precision used for computations.} distinguish between any pair $x_1,\,x_2\in\mathbb{R}$: $x_1\ne x_2$. To see whether the constant behavior of $\sigma_\mathrm{GW}$ indeed causes problems, we look at the pre-activated values $x\cdot w+b$ for $x$ from the validation split in the synthetic task of approximating $x\mapsto|x|$, $x\in[-\pi,\pi]$. The histogram of these pre-activated values for the $f_\mathrm{GW}$ with hidden layer size $n=100$ is given in Figure~\ref{pre_activated}.
\begin{figure}[htbp]
\begin{center}
    \includegraphics[width=.6\textwidth]{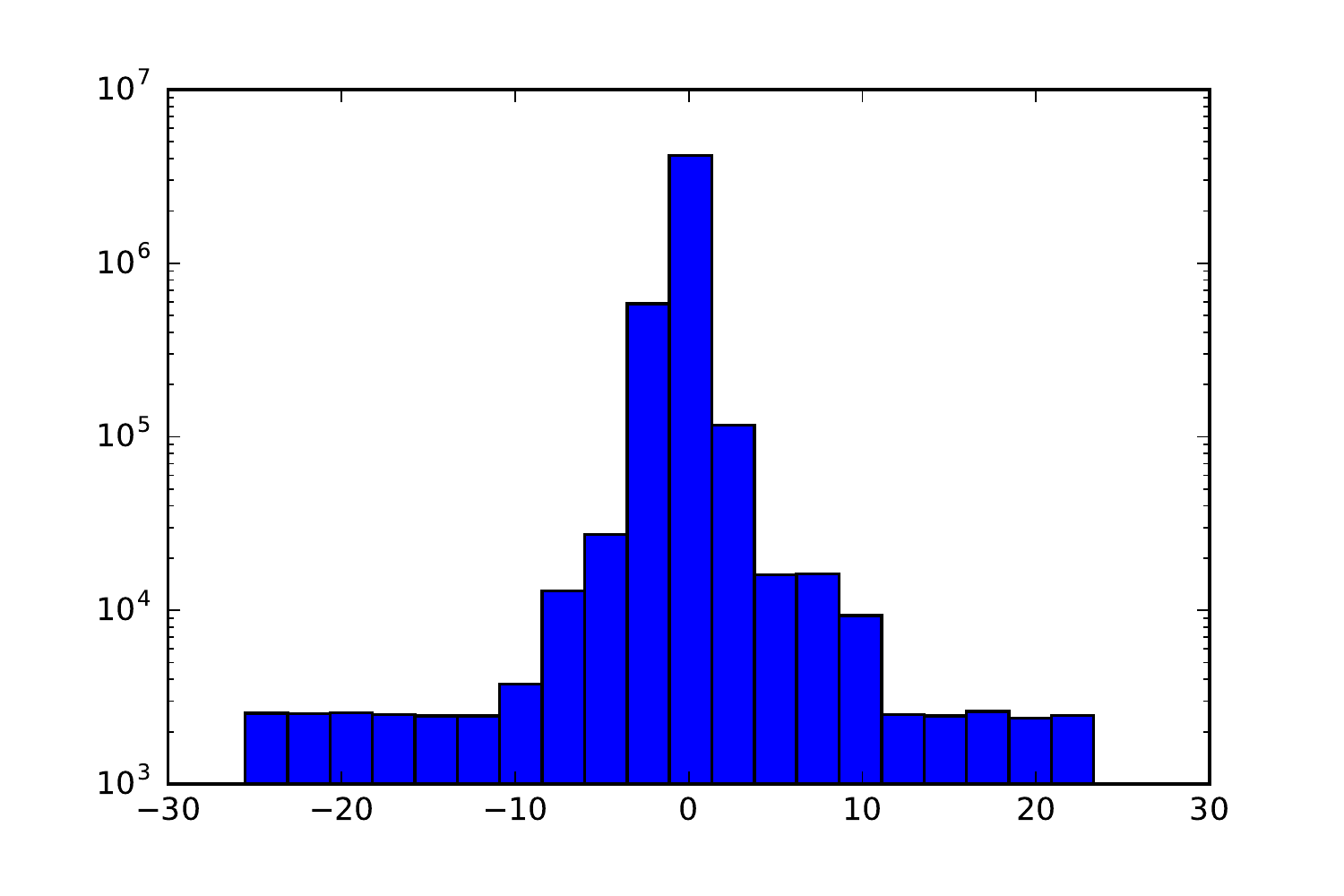}
\end{center}
\caption{Histogram of pre-activated values ($x\cdot w+b$) in the FNN of \citet{gallant1988there}. Frequencies are at log-scale.}
\label{pre_activated}
\end{figure}
It turns out that $\approx 8\%$ of pre-activated values are outside of $[-\pi/2,\pi/2]$, and this information is lost when filtered through $\sigma_\mathrm{GW}$.

\section{Conclusion and Future Work}
All Fourier neural networks are not better than the standard neural network with sigmoid activation except when it comes to modeling synthetic data. The architecture of \citet{gallant1988there} is the best among Fourier neural networks. When the function being approximated is known and depends on multiple variables, the neural networks with just one hidden layer may provide much better approximation compared to truncated Fourier series.

In this paper we focused on neural architectures with one hidden layer. It is interesting to compare Fourier neural networks in a multilayer setup. We defer such study to our future work which will also include experiments with a larger variety of functions, as well as mathematical analysis of the approximation of Barron functions by Silvescu's and Liu's Fourier neural networks.

\section*{Compliance with ethical standards}
\textbf{Conflict of Interest.} The authors declare that they have no conflict of
interest.

\bibliographystyle{plainnat}  
\bibliography{references}

\end{document}